\documentclass{article}
\usepackage{spconf,amsmath,graphicx,amsthm,amsfonts,thmtools,thm-restate}
\usepackage{cite}
\usepackage{enumitem}
\usepackage[citecolor=blue, colorlinks=true]{hyperref}

\ninept
\newcommand{\matr}[1]{\mathbf{#1}}
\usepackage{yhmath}

\usepackage{pgfplots}
\usetikzlibrary{arrows,shapes,calc,tikzmark,backgrounds,matrix,decorations.markings}
\usepgfplotslibrary{fillbetween}

\pgfplotsset{compat=1.3}

\usepackage{relsize}
\tikzset{fontscale/.style = {font=\relsize{#1}}
    }

\definecolor{lavander}{cmyk}{0,0.48,0,0}
\definecolor{violet}{cmyk}{0.79,0.88,0,0}
\definecolor{burntorange}{cmyk}{0,0.52,1,0}

\definecolor{asuorange}{rgb}{1,0.699,0.0625}
\definecolor{asured}{rgb}{0.598,0,0.199}
\definecolor{asuborder}{rgb}{0.953,0.484,0}
\definecolor{asugrey}{rgb}{0.309,0.332,0.340}
\definecolor{asublue}{rgb}{0,0.555,0.836}
\definecolor{asugold}{rgb}{1,0.777,0.008}

\DeclareUnicodeCharacter{2212}{−}
\usepgfplotslibrary{groupplots,dateplot}
\usetikzlibrary{patterns,shapes.arrows}
\pgfplotsset{compat=newest}

\usepackage{tikz,pgfplots}
\usepgfplotslibrary{groupplots}

\makeatletter
\usepackage{steinmetz}

\newcommand{\e}{\mathbf{e}}

\renewcommand{\v}{\mathbf{v}}

\newcommand{\y}{\mathbf{y}}
\newcommand{\x}{\mathbf{x}}

\newcommand{\X}{\mathbf{X}}
\renewcommand{\L}{\mathbf{L}}

\newcommand{\E}{\mathbf{E}}

\newcommand{\A}{\mathbf{A}}
\newcommand{\V}{\mathbf{V}}

\newcommand{\Z}{\mathbf{Z}}

\newcommand{\bbR}{\mathbb{R}}
\newcommand{\bbE}{\mathbb{E}}

\usepackage{extarrows}

\usepackage{xcolor}
\usepackage{psfrag,framed}
\usepackage{lipsum}
\usepackage{chapterbib}
\PassOptionsToPackage{normalem}{ulem}
\usepackage{ulem}
\definecolor{shadecolor}{RGB}{220,220,220}

\newcommand{\bm}{\boldsymbol}

\newcommand{\Lfull}{\mathbf{L}}
\newcommand{\Lpart}{\mathbf{L}_{p}}
\makeatother

\usepackage{shortcuts}

\newtheorem{Lemma}{Lemma}

\newtheorem{assumption}{A\!\!}
\newtheorem{definition}{Definition}

\title{Learning Graph from Smooth Signals under Partial Observation:\\A Robustness Analysis}
\twoauthors
 {Hoang-Son Nguyen\sthanks{The first author performed part of the work while at Department of SEEM, The Chinese University of Hong Kong.}}
	{School of EECS, Oregon State University}
 {Hoi-To Wai}
	{Dept.~of SEEM, CUHK}
\begin{document}

% ------ Sean's renewcommand ------ %
\renewcommand{\I}{\mathbf{I}}
\renewcommand{\D}{\mathbf{D}}
\renewcommand{\deg}{\text{deg}}
\newcommand{\Diag}{\text{Diag}}
% --------------------------------- %

\maketitle

\begin{abstract}
Learning the graph underlying a networked system from nodal signals is crucial to downstream tasks in graph signal processing and machine learning. The presence of hidden nodes whose signals are not observable might corrupt the estimated graph. While existing works proposed various robustifications of vanilla graph learning objectives by explicitly accounting for the presence of these hidden nodes, a robustness analysis of ``naive'', hidden-node agnostic approaches is still underexplored. This work demonstrates that vanilla graph topology learning methods are implicitly robust to partial observations of low-pass filtered graph signals. We achieve this theoretical result through extending the restricted isometry property (RIP) to the Dirichlet energy function used in graph learning objectives. We show that smoothness-based graph learning formulation (e.g., the GL-SigRep method) on partial observations can recover the ground truth graph topology corresponding to the observed nodes. Synthetic and real data experiments corroborate our findings.
\end{abstract}

\section{Introduction}\vspace{-.2cm}
% This work studies the problem of graph topology learning from a set of graph signal observations. 
As a crucial first step before downstream tasks such as graph machine learning, graph data denoising, etc.,  graph topology learning has received attention from the signal processing community \cite{dong2016smooth, kalofolias2016learn, dong2019learning, mateos2019connecting}. 
A typical setup in the literature is to consider \emph{full observation} data where the graph signals on each node are recorded simultaneously. However, for graphs with large number of nodes or even open systems, obtaining full observation of graph signals can be challenging. For social networks, this requires each individual to report their states within a limited time interval; for biological or physical networks, this requires estimating the states of each components. 
As such, we are often restricted to accessing \emph{partially observed} graph signals, where the states of some nodes become unobservable \cite{matta2020partialobs}. 
% These unobserved nodes are \emph{hidden} from us whose existence may not even be known.

Although hidden nodes are not observed, these nodes may still influence the observed nodes in the network system, and eventually this may affect the reliability of the graph topology learning pipelines.  Recent works have considered methods for graph topology learning that specifically accounts for the influence of these hidden nodes. Importantly, \cite{chandrasekaran2012latent} found that the unknown influences are low rank provided that the number of hidden nodes is smaller than the number of observed nodes. This observation is then leveraged to develop a low rank influences-aware formulation for Gaussian graphical model inference. Subsequent works that leveraged the low-rank structure include \cite{anandkumar2013learning} which considered learning graphs with local tree-like structure, \cite{buciulea2022learning, peng2024network, navarro2024joint, berger2020efficient} which studied alternative observation models with stationary and/or smooth graphs signals.
% This inspired recent works to propose graph topology learning methods that are aware of the hidden nodes for Gaussian graphical model inference \cite{chandrasekaran2012latent}, stationary graph signals via spectral template \cite{buciulea2022learning}, smooth signals \cite{buciulea2022learning}, etc. 

While various alternative formulations were proposed for explicitly robustifying graph learning against partial observation, these approaches may entail additional computation complexity in solving the corresponding optimization problem. On the other hand, investigations on the implicit robustness of existing methods, i.e., directly applying \cite{dong2016smooth, kalofolias2016learn} on the partially observed signals, are much less explored. We raise the research question --- \textit{are the graph learning formulations in \cite{dong2016smooth, kalofolias2016learn}, such as GL-SigRep, implicitly robust to partial observation?}
We give a positive answer to the above question by analyzing the implicit robustness of the GL-SigRep based graph learning objective \cite{dong2016smooth} utilizing the Dirichlet energy of graph signals, under partial observations. 

Towards this end, we notice a similar line of works in \cite{matta2020partialobs, santos2024learning} that studied the implicit robustness of graph learning methods based on covariance estimation. In contrast, our work studies the more challenging optimization-based graph learning method and utilizes the notion of \textit{low-pass graph signals} \cite{ramakrishna2020userguide} to develop our analysis.
Our contributions are summarized as:
\begin{itemize}[leftmargin=*]
\item We show that if the graph signals are sufficiently low pass, then the GL-SigRep objective function is insensitive to the effects of partial observation. As a consequence, an optimal solution obtained from GL-SigRep on {\it partially observed graph signals} should achieve similar objective value as an optimal solution obtained from GL-SigRep on {\it fully observed graph signals}, and vice versa. In other words, the former shall learn a graph topology that is similar to one taken from learning with full observation. 
\item Our analysis is based on viewing the GL-SigRep objective function with partial observations as a quadratic form with randomly sampled entries. With sufficiently low-pass signals, this mimics the scenario of sparse signal recovery from a small number of observations. Using this insight, we develop RIP-based bounds to justify the implicit robustness of GL-SigRep.
    % \item If the graph signals are generated from a low-pass filtering process and there are enough observed nodes, then the ``naive'' smoothness-based graph learning approach is guaranteed to learn a partial graph topology similar to one taken from learning with full observation; this holds even for $n \ll N$.
    % \item We show the above theoretical result through utilizing a new perspective: smoothness structure implies low-dimensionality of the graph signals, which can be exploited by a RIP-based argument for quadratic forms of the sampled graph Laplacian.
\item We conduct numerical experiments to validate our theories on synthetic data, and show its applicability in real datasets.
\end{itemize}

\textbf{Notations.} For a matrix $\matr{X}$, we let $[\matr{X}]_{ij}$ be its $(i, j)$-th entry. $||\matr{X}||_{2}$ denotes its spectral norm, and its maximum singular value is $\sigma_{\max}(\matr{X}) = ||\matr{X}||_{2}$, whereas its smallest non-zero singular value is $\sigma_{\min}^{+}(\matr{X})$. For a vector $\matr{x}$, its $\ell_{p}$ norm is denoted by $||\matr{x}||_{p}$.

\vspace{-.1cm}
\section{Preliminaries}\vspace{-.1cm}
We consider an $N$-node weighted, undirected, connected, simple graph ${\cal G} = ({\cal V}, {\cal E})$, where the node set is ${\cal V} = \{1,\ldots,N\}$ and the edge set is ${\cal E} \subseteq {\cal V} \times {\cal V}$. The graph is also endowed with a symmetric weighted adjacency matrix $\A \in \mathbb{R}^{N \times N}_{+}$ and a Laplacian matrix $\L = \D - \A$, where $\D = {\rm Diag}( \A {\bf 1} ) \in \mathbb{R}^{N \times N}$ is a diagonal matrix corresponding to node degrees. The Laplacian matrix admits an eigendecomposition $\L = \V \mathbf{\Lambda} \V^\top$ such that $\mathbf{\Lambda} = {\rm Diag}( \lambda_1, \ldots, \lambda_N )$ has eigenvalues, also known as the graph frequencies, sorted in ascending order $0 = \lambda_1 \leq \lambda_2 \leq \cdots \leq \lambda_N$, and the corresponding eigenvectors $\matr{V} = [\v_1, \v_2, ..., \v_{N}]$ collected in the columns of $\matr{V}$.

A graph signal on $\mathcal{G}$ is a vector $\matr{y} = [y_1, y_2, ..., y_N] \in \bbR^{N}$, where $y_i$ is the scalar signal of node $i \in \mathcal{V}$. As a common practice in the literature, we consider graph signals that can be modeled as the output of a graph filter $\mathcal{H}(\matr{L})$ as
\begin{align} \label{eq:graph-signal}
    \matr{y} = \mathcal{H}(\matr{L})\matr{x},
\end{align}
where $\matr{x} \in \bbR^{N}$ is an excitation signal. The graph filter $\mathcal{H}(\matr{L})$ is a polynomial of $\matr{L}$ with filter coefficients $\{h_t\}_{t=0}^{T}$,
\begin{align}
    \textstyle \mathcal{H}(\matr{L}) = \sum_{t=0}^{T}h_t\matr{L}^{t} = \matr{V}h(\matr{\Lambda})\matr{V}^{\top},
\end{align}
where the latter defines the frequency response function $h(\lambda) = \sum_{t=0}^{T}h_t\lambda^t$ and $h(\matr{\Lambda}) = \diag(h(\lambda_1), ..., h(\lambda_N))$. In this way, $\matr{y}$ can also be regarded as the output of a network process characterized in the frequency domain by $h(\cdot)$. We impose the following regularity assumptions regarding this signal model.
\begin{assumption} \label{as:regularity}
The excitation signal satisfies $||\matr{x}||_{2} \leq M$, and the graph filter satisfies $\max_{i}|h(\lambda_i)| \leq H$.
\end{assumption}
An important class of filter is the low-pass graph filter \cite{ramakrishna2020userguide}, defined as following.

\begin{definition} \label{def:lpf}
A graph filter $\mathcal{H}(\cdot)$ is $K$-low-pass if
\begin{align*}
    \textstyle \eta_{K} := \frac{\max_{i=K+1, ..., N}|h(\lambda_i)|}{\min_{j=1, ..., K}|h(\lambda_j)|} < 1,
\end{align*}
where $K \leq N$ is bandwidth, and $\eta_K$ is sharpness ratio of $\mathcal{H}(\cdot)$.
\end{definition}
By definition, a low-pass graph filter retains the energy of excitation signal at low frequencies, and attenuates those at high frequencies. A low-pass graph signal is the output of a low-pass graph filter, which is found to be prevalent in real-world network dynamics \cite{ramakrishna2020userguide}. Particularly for the sufficiently low pass graph signals with $\eta_K \ll 1$, an important feature is its smoothness with respect to the underlying graph $\mathcal{G}$: the scalar signals of connected nodes tend to be similar \cite{ramakrishna2020userguide}. This property of low-pass signals is exploited in many graph signal processing (GSP) applications \cite{ramakrishna2020userguide}, including community detection \cite{wai2019blind, wai2022partial}, central node identification \cite{he2022central, he2023central}, and graph topology learning.

% \vspace{.1cm}
% \noindent
% \textbf{Graph Topology Learning.} 
\subsection{Graph Topology Learning}
To preface the main analytical results of this paper, we conclude the section by formally stating the graph topology learning problems under full and partial observations.  

A consequence with the prevalence of low-pass graph signals is that the low-pass property has enabled works on graph topology learning \cite{dong2016smooth, mateos2019connecting} using the smoothness criterion. We observe that the graph quadratic form \cite{shuman2013emerging} admits the following expansion:
\begin{align}
\textstyle \matr{y}^{\top}\matr{L}\matr{y} = \sum_{i=1}^{N}\sum_{j=1}^{N}[\matr{A}]_{ij}(y_i - y_j)^{2} \geq 0,
\end{align}
where it evaluates the sum of squared difference in scalar signals between every pair of connected nodes. With this in mind, the graph signals $\y_1, \ldots, \y_M \in \mathbb{R}^{N}$ are said to be smooth with respect to (w.r.t.)~the Laplacian matrix $\matr{L}$ if $\y_m^\top \L \y_m \approx 0$ for all $m$. 
% The goal of this work is to study the graph topology learning problem \cite{dong2016smooth, mateos2019connecting} that learns the graph Laplacian matrix from low-pass signals. The graph quadratic form \cite{shuman2013emerging} provides a measure of the smoothness of a signal $\matr{y}$ on $\mathcal{G}$, defined as
% \begin{align}
%     \textstyle \matr{y}^{\top}\matr{L}\matr{y} = \sum_{i=1}^{N}\sum_{j=1}^{N}[\matr{A}]_{ij}(y_i - y_j)^{2} \geq 0,
% \end{align}
% which is a sum of squared difference in scalar signals between every pair of connected nodes. In the smoothness-based graph learning framework, we assume that the graph signals $\y_1, \ldots, \y_M \in \mathbb{R}^{N}$ are smooth with respect to $\L$ such that
% \begin{align}
%     \y_m^\top \L \y_m \approx 0,~\forall m=1,\ldots,M.
% \end{align}
% It was shown that many real-world graph data such as nodal signals on social network, power network, financial network, etc. exhibit such a smoothness pattern with respect to the underlying network topology \cite{ramakrishna2020userguide}.

This observation has inspired the smoothness-based graph learning problem \cite{dong2016smooth, kalofolias2016learn}, which minimizes the graph quadratic form. We consider the GL-SigRep formulation as \cite{dong2016smooth}:
\begin{align} \textstyle
\min_{ \hat{\L} \in {\cal L}_N } J_f( \hat{\L} ) := \sum_{m=1}^M \y_m^\top \hat{\L} \y_m + \frac{\lambda}{2} \| \hat\L \|_F^2, \tag{Full GL-SigRep} \label{eq:gl}
\end{align}
where $\lambda \geq 0$ is the regularization parameter, ${\cal L}_{N} \subseteq \mathbb{R}^{N \times N}$ is the set of normalized Laplacian matrix
\begin{align*}
{\cal L}_N := \{ \hat{\L} : {\rm Tr}(\hat{\L}) = N, \hat{\L} {\bf 1} = {\bf 0}, \hat{\L} - \Diag(\hat{\L}) \leq \matr{0}, \hat{\L} = \hat{\L}^\top \}. \nonumber
\end{align*}
We denote $\L^\star$ as an optimal solution to the above problem. Note that regularizers such as degree log-barrier \cite{kalofolias2016learn} may also be included. 
% We focus on the unregularized solution to improve the interpretability of our results. Nonetheless, our analysis is without loss of generality, and can be extended to accommodate objectives with various regularizers.

% The above graph learning formulation \eqref{eq:gl} requires \emph{full observations} of graph signals. However, full observation of graph signals are not always available in practice, due to real-world constraints on data collection. 
In the case of {\it partial observation}, one only observes the signals of a given subset $\mathcal{V}_o \subseteq \mathcal{V}$ of $|\mathcal{V}_o| = n$ nodes from the original graph, where $n \leq N$. The set of partially observed graph signals $\y_{o,1}, \ldots, \y_{o,M} \in \bbR^{n}$ can be described as:
\begin{align*}
    \textstyle \y_{o,m} = \E_o \y_m,~\text{with}~\E_o \in \{0, 1\}^{n \times N}.
\end{align*}
% Here, the binary matrix $\E_o$ represents the \textit{uniform sampling without replacement} to choose the observed nodes. 
In particular, $\E_{o}$ masks out part of the original signals $\matr{y}_{m} \in \bbR^{N}$ that are not observed. Formally, each columns of $\matr{E}_{o} \in \{0, 1\}^{n \times N}$ only has one ``$1$'' entry, and that
\begin{align*}
    \textstyle [\matr{E}_o]_{ij} = \begin{cases}
        1 & \text{if node $j$ is observed}\\
        0 & \text{if node $j$ is unobserved}
    \end{cases}
\end{align*}

To learn the subgraph induced by ${\cal V}_o$, we consider the following GL-SigRep graph learning problem:
\begin{align}
    \textstyle \min_{ \hat{\L}_p \in {\cal L}_n } J_p( \hat{\L}_p ) := \textstyle\sum_{m=1}^M \y_{o,m}^\top \hat{\L}_p \y_{o,m} + \frac{\lambda}{2} \| \hat\L_p \|_F^2. \tag{Partial GL-SigRep} \label{eq:gl_p}
\end{align}
We denote $\matr{L}_p^\star$ as an optimal solution to the above problem.
Notice that \eqref{eq:gl_p} is identical to \eqref{eq:gl} except for the use of partial graph signals. In other words, \eqref{eq:gl_p} considers a scenario when the graph learning task is solved using the plain GL-SigRep while being {\it agnostic} to the presence of hidden nodes. 
% and $\hat{\L}_p$ is an estimate for the Laplacian matrix corresponding to a subgraph of ${\cal G}$ with nodes selected via $\E_o$. We denote $\L_p^\star$ as an optimal solution to \eqref{eq:gl_p}.

The primary goal of this paper is to compare the solutions $\matr{L}^{\star}$ and $\matr{L}_p^\star$. Towards this goal, the next section shall analyze the solutions to \eqref{eq:gl} and \eqref{eq:gl_p} and demonstrate that the graph topology learning can be agnostic to hidden nodes if the graph signals are sufficiently low pass. 
% we investigate a theoretical relationship between optimal solution $\matr{L}^{\star} \in \bbR^{N \times N}$ of the full observation formulation \eqref{eq:gl} with optimal solution $\matr{L}^{\star}_{p} \in \bbR^{n \times n}$ of the partial observation formulation.

\vspace{-.2cm}
\section{Main Result}\vspace{-.2cm}
This section establishes an invariance relation between the optimal solutions to \eqref{eq:gl} and \eqref{eq:gl_p}. Particularly, we wish to show that $\L_p^\star$ achieves similar objective value as a column/row sampled version of $\L^\star$, and vice versa. For analysis purpose, we define the following surrogate solutions:
\begin{align}
    &\textstyle \widehat{\L} := \frac{N}{n} \E_o^\top \L_p^\star \E_o,\\
    &\textstyle \widetilde{\matr{L}}_{p} := \frac{n}{{\rm Tr}( \Lfull_{oo}^\star ) + {\bf 1}^\top \Lfull_{oh}^\star {\bf 1}} ( {\bf E}_o \Lfull^\star {\bf E}_o^\top + {\rm Diag}( \Lfull_{oh}^\star {\bf 1}) ).
\end{align}
Observe that $\widehat{\L} \in {\cal L}_N$ is rescaled from $\matr{L}_p^\star$ where the hidden nodes related entries are set to zero. Similarly, we have constructed $\widetilde{\L}_p \in {\cal L}_n$ from $\L^\star$ to represent the observable subgraph. Additionally, we define $\matr{L}^{\star}_{oo} := {\bf E}_o \Lfull^\star {\bf E}_o^\top$, $\L_{oh}^\star$ as the submatrices of $\L^\star$ corresponding to edges within the observable nodes, and between observable and hidden nodes, respectively. 

Throughout this section, we consider the unregularized case of \eqref{eq:gl}, \eqref{eq:gl_p} with $\lambda = 0$ for ease of presentation. Nonetheless, our analysis is without loss of generality, and can be extended to accommodate objectives with various regularizers.
% Rescaling and adding the term ${\rm Diag}(\matr{L}^{\star}_{oh}\matr{1})$ into $\matr{L}^{\star}_{oo}$ corrects the degree of the observed nodes in $\matr{L}^{\star}_{oo} \notin \mathcal{L}_{n}$ by removing the connections to hidden nodes. 
We impose the two conditions regarding $\matr{L}_{oh}^{\star}$ and $\matr{L}^{\star}_{oo}$:
\begin{assumption} \label{as:1}
$-\L_{oh}^\star {\bf 1} \leq \epsilon {\bf 1}$ for a sufficiently small $\epsilon > 0$.
\end{assumption}
\begin{assumption} \label{as:2}
${\rm Tr}( \Lfull_{oo}^\star ) + {\bf 1}^\top \L_{oh}^\star {\bf 1} \geq c n$ for a sufficiently large $c>0$.
\end{assumption}
\noindent Assumption A\ref{as:1} requires the number of edges between any observed node to the hidden nodes to be smaller than a constant $\epsilon$. The left-hand expression in A\ref{as:2} stands for the total degrees of the subgraph of observed nodes, and it is required to be sufficiently large compared to the number of observed nodes $n$.

Our main results will be presented for two cases in sequel depending on the conditions for low pass graph signals. 

\vspace{.1cm}
\noindent 
{\bf I.~Ideal Scenario.} We first consider a case where the fully observed graph signals $\y_m$ lie in a $K$-dimensional subspace spanned by $\matr{V}_{K} = [\v_1, \v_2, ..., \v_{K}] \in \bbR^{N \times K}$, where $K \ll n \leq N$.
\begin{assumption} \label{as:spanVK}
For any signal realization $m = 1, ..., M$, $\y_m \in {\rm span} ( \V_K )$.
\end{assumption}
\noindent Note that A\ref{as:spanVK} is satisfied when the data generation process involves an ideal low-pass graph filter with a cutoff bandwidth of $K$. 
% solving \eqref{eq:gl} learns the ground truth graph topology, i.e., the graph signal is indeed smooth with respect to underlying graph. This is a plausible scenario for graph signals that are resulted from an \textit{ideal low-pass graph filtering process} with cutoff bandwidth $K$ such that $h(\lambda_{K+1}) = ... = h(\lambda_{N}) = 0$, or equivalently $\eta_{K} = 0$ \cite{ramakrishna2020userguide}.
Under A\ref{as:spanVK}, we obtain our first main result comparing between $\L^\star$ and $\L_p^\star$:
\begin{restatable}{Theorem}{mainresult}
\label{thm:main}
Consider a set $\mathcal{V}_{o} = \{s(1), s(2), ..., s(n)\}$ of partial observation, which is sampled uniformly without replacement from the node set $\mathcal{V}$. Suppose Assumption \ref{as:1}, \ref{as:2}, and \ref{as:spanVK} hold. Then, with any $\delta \in (0, 1)$, there is $t \in (0, 1)$ such that with probability at least $1 - \delta$, provided that the number of observations satisfies
\begin{align} \label{eq:nNcond}
    \textstyle \frac{n}{N} \geq \frac{3}{t^2}\max_{1 \leq i \leq N}\|\V_{K}^{\top}\e_{i}\|_{2}^{2}\ln\left(\frac{K}{\delta}\right),
\end{align}
the following inequalities hold
\begin{align}
    &\textstyle J_{p}( \Lpart^\star ) \leq J_{p} ( \widetilde{\L}_{p} ) \leq \frac{1+t}{c} \frac{\sigma_{\max}(\L)}{\sigma_{\min}^{+}(\L)} \, J_{p}( \Lpart^\star ) + {\cal O} \left( \frac{\epsilon}{ c } \right), \text{ and} \label{thm:part1}\\
    &\textstyle J_{f}( {\Lfull}^\star ) \leq J_{f}( \widehat{\Lfull} ) \leq \frac{1+t}{c} \frac{\sigma_{\max}(\L)}{\sigma_{\min}^{+}(\L)} \, J_{f}( \Lfull^\star ) + {\cal O} \left( \frac{N\epsilon}{ nc } \right). \label{thm:part2}
\end{align}
\end{restatable}
\noindent The \eqref{eq:nNcond} depends on $\max_{1 \leq i \leq N}||\matr{V}_{K}^{\top}\matr{e}_{i}||_{2}$ and the bandwidth $K$. As the graph signals become more low-pass, the bandwidth $K$ is lower and $||\matr{V}_{K}^{\top}\matr{e}_{i}||_{2}$ becomes smaller, which implies that the condition is more likely to be satisfied.
Now, suppose that {$\frac{1+t}{c} \frac{\sigma_{\max}(\L)}{\sigma_{\min}^{+}(\L)} = \Theta(1)$} and A\ref{as:1}, A\ref{as:2} hold, Theorem \ref{thm:main} implies that $J_{p}(\widetilde{\matr{L}}_{p})$ (resp. $J_{f}(\hat{\matr{L}})$) approximates $J_{p}(\matr{L}^{\star}_{p})$ (resp. $J_{f}(\matr{L}^{\star})$). 
% In this case, such condition can be satisfied for non-modularized graphs.
Importantly, our result is insensitive to the number of hidden nodes $N-n$ provided that the condition \eqref{eq:nNcond} holds.

\vspace{.1cm}
\noindent
\emph{Proof Insight:} The proof of Theorem \ref{thm:main} is similar to results from compressed sensing \cite{foucart2013cs} and matrix sketching \cite{woodruff2014sketch} literature, which is also applied in other graph signal processing contexts \cite{puy2018sampling}. Specifically, our theorem is achieved by establishing the following one-sided restricted isometry property (RIP) that relates the partial graph quadratic form $\matr{y}_{o}^{\top}(\matr{E}_{o}\matr{L}\matr{E}_{o}^{\top})\matr{y}_{o}$ to the full one $\matr{y}^{\top}\matr{L}\matr{y}$: with high probability and for any $\y \in {\rm span}( \V_K )$,
\beq \label{eq:rip_ideal}
\textstyle \y_{o}^{\top} (\E_o \L \E_o^\top) \y_{o} \leq (1+t)\frac{n}{N}\frac{\sigma_{\max}(\L)}{\sigma_{\min}^{+}(\L)} \y^\top \L \y.
\eeq 
A detailed derivation of this result can be found in the \href{https://www1.se.cuhk.edu.hk/~htwai/pdf/icassp26-partialobs.pdf}{Online Appendix}, along with the proof of Theorem \ref{thm:main}. \hfill {\bf Q.E.D.}

\vspace{.2cm}
\noindent
\textbf{II.~Non-ideal Scenario.}
In the case there are some realizations $\matr{y} \in \{\matr{y}_1, ..., \matr{y}_{m}\}$ such that $\matr{y} \notin {\rm span}(\matr{V}_{K})$, we decompose this graph signal into two components, $\mathbf{y}^{\parallel} \in {\rm span}\{\mathbf{V}_{K}\}$ lying on ${\rm span}\{\mathbf{V}_{K}\}$ and $\mathbf{y}^{\perp}$ being orthogonal to the said subspace:
\begin{align}
    \textstyle \mathbf{y} = \mathbf{y}^{\parallel} + \mathbf{y}^{\perp}, \label{eq:nonideal_model}
\end{align}
with $\matr{y}^{\parallel} := \matr{V}_{K}(\matr{V}_{K}^{\top}\matr{V}_{K})^{-1}\matr{V}_{K}^{\top}\matr{y} = \matr{V}_{K}\matr{V}_{K}^{\top}\matr{y}$ and $\matr{y}^{\perp} := \matr{y} - \matr{y}^{\perp} = (\matr{I} - \matr{V}_{K}\matr{V}_{K}^{\top})\matr{y}$. We aim to show that as long as the residue component $\matr{y}^{\perp}$ is sufficiently small compared to the low-frequency component $\matr{y}^{\parallel}$, then the main result still holds. Towards this goal, we give a bound that relates the graph quadratic form with respect to $\mathbf{y}_{o} \in \bbR^{n}$ (i.e., $\mathbf{y}_{o}^{\top}\mathbf{L}\mathbf{y}_{o}$) by the graph quadratic form of $\mathbf{y}_{o}^{\parallel} := \matr{E}_{o}\matr{y}^{\parallel} \in \bbR^{n}$ (i.e., $(\mathbf{y}^{\parallel}_{o})^{\top}\mathbf{L}\mathbf{y}_{o}^{\parallel}$), plus some residual error that is dependent on the properties of the filtering process $\mathcal{H}(\matr{L})$ that underlies the graph signal $\matr{y}$. Specifically, with the non-ideal low-pass signal model as in \eqref{eq:nonideal_model}, one can show that
\begin{align*}
    \textstyle \matr{y}_{o}^{\top}(\matr{E}_{o}\matr{L}\matr{E}_{o}^{\top})\matr{y}_{o} \leq (\matr{y}_{o}^{\parallel})^{\top}(\matr{E}_{o}\matr{L}\matr{E}_{o}^{\top})\matr{y}^{\parallel}_{o} + \textstyle \mathcal{O}(||\matr{L}||_{2}\eta_{K}H^{2}M^{2}).
\end{align*}

With the aforementioned observation, we can establish a RIP-like result akin to \eqref{eq:rip_ideal}: with high probability, for any signal $\matr{y} \in \bbR^{N}$:
\begin{align*}
    \textstyle \matr{y}_{o}^{\top}(\matr{E}_{o}\matr{L}\matr{E}_{o}^{\top})\matr{y}_{o} &\leq (1+t)\textstyle\frac{n}{N}\frac{\sigma_{\max}(\matr{L})}{\sigma_{\min}^{+}(\matr{L})}\matr{y}^{\top}\matr{L}\matr{y} + \textstyle \mathcal{O}(||\matr{L}||_{2}\eta_{K}H^{2}M^{2}),\label{ineq:nonideal}
\end{align*}
from which an extended result of Theorem \ref{thm:main} can be attained for non-ideal low-pass graph signals. The proof follows similarly from that of Theorem \ref{thm:main}:

\begin{restatable}{Corollary}{cornonideal}
\label{cor:nonideal}
Consider the same setting as Theorem \ref{thm:main}, except the ideal low-pass-ness of signal as in A\ref{as:spanVK}. Suppose A\ref{as:regularity} holds. Then, with any $\delta \in (0, 1)$, there is an $t \in (0, 1)$ such that with probability at least $1 - \delta$,
% provided that
% \begin{align*}
%     \textstyle \frac{n}{N} \geq \frac{3}{t^2}\max_{1 \leq i \leq N}\|\V_{K}^{\top}\e_{i}\|_{2}^{2}\ln\left(\frac{K}{\delta}\right),
% \end{align*}
the following inequalities hold
\begin{align*}
    &\textstyle J_{p}( \Lpart^\star ) \leq J_{p} ( \widetilde{\L}_{p} ) \leq \frac{1+t}{c} \frac{\sigma_{\max}(\L)}{\sigma_{\min}^{+}(\L)} \, J_{p}( \Lpart^\star ) + {\cal O} \left( \frac{||\L||_{2}\eta_{K}H^2M^2 + \epsilon}{ c } \right),\\
    &\textstyle J_{f}( {\Lfull}^\star ) \leq J_{f}( \widehat{\Lfull} ) \leq \frac{1+t}{c} \frac{\sigma_{\max}(\L)}{\sigma_{\min}^{+}(\L)} \, J_{f}( \Lfull^\star ) + {\cal O} \left(\frac{N(||\L||_{2}\eta_{K}H^2M^2 + \epsilon)}{nc}\right),
\end{align*}
provided that the number of observations satisfies \eqref{eq:nNcond}.
\end{restatable}
In Corollary \ref{cor:nonideal}, besides the terms that appear in Theorem \ref{thm:main}, we now have an additional error term ${\cal O}(||\L||_{2}\eta_{K}H^{2}M^{2}/c)$. We note that this new error term is dependent on a constant upperbound $H$ of the frequency response and a constant bound $M$ of the excitation signal energy, as assumed in A\ref{as:regularity}. Importantly, the new additive error also depends on the graph filter's sharpness ratio $\eta_{K}$, which can quantify how much of the graph signal lies on ${\rm span}(\V_{K})$. As the signal is more low-pass and $\eta_{K}$ gets close to $0$, its energy becomes more concentrated in ${\rm span}(\V_{K})$, and the new error term decreases.

To conclude, the results presented in this section have an important implication for graph learning from partially observed graph signals: if the signals are sufficiently smooth, the graph underlying the observed nodes can still be reliably estimated from the partial signals $\matr{y}_{o, 1}, ..., \matr{y}_{o, m}$ via the agnostic formulation \eqref{eq:gl_p}, and the estimation $\matr{L}^{\star}_{p}$ would be close to the corresponding part of the original solution $\matr{L}^{\star}$ of \eqref{eq:gl}. Therefore, the graph learning criterion \eqref{eq:gl} is implicitly robust to partial observation to a certain extent based on the smoothness level of graph signals.

\section{Numerical Results}\vspace{-.2cm}
This section presents three sets of numerical experiments on both synthetic and real data to validate our theoretical findings.

\begin{figure}[t!]
\centering
  {\sf \resizebox{1.\linewidth}{!}{% This file was created with matplot2tikz v0.4.0.
\begin{tikzpicture}

\definecolor{darkcyan32144140}{RGB}{32,144,140}
\definecolor{darkgray176}{RGB}{176,176,176}
\definecolor{darkslateblue48103141}{RGB}{48,103,141}
\definecolor{darkslateblue6857130}{RGB}{68,57,130}
\definecolor{gold25323136}{RGB}{253,231,36}
\definecolor{lightgray204}{RGB}{204,204,204}
\definecolor{mediumseagreen53183120}{RGB}{53,183,120}
\definecolor{yellowgreen14421467}{RGB}{144,214,67}

\begin{groupplot}[group style={group size=2 by 1}]
\nextgroupplot[
legend cell align={left},
legend style={
  fill opacity=0.8,
  draw opacity=1,
  text opacity=1,
  at={(0.97,0.03)},
  anchor=south east,
  draw=lightgray204
},
tick align=outside,
tick pos=left,
title={\Large Median F1 Scores},
x grid style={darkgray176},
xlabel={\Large Number of Observations $n$},
xmajorgrids,
xmin=8, xmax=52,
xtick style={color=black},
ticklabel style={font=\large},
y grid style={darkgray176},
ylabel={},
ymajorgrids,
ymin=0.18135593220339, ymax=1.03898305084746,
ytick style={color=black}
]
\addplot [semithick, darkslateblue6857130, mark=*, mark size=3, mark options={solid}]
table {%
10 0.233926128590971
15 0.220338983050847
20 0.281152955790187
25 0.395535381171496
30 0.622306199314921
35 0.8070172066837
40 0.919037199124726
45 0.976190476190476
50 1
};
\addlegendentry{$\alpha=0.5$}
\addplot [semithick, darkslateblue48103141, mark=*, mark size=3, mark options={solid}]
table {%
10 0.269230769230769
15 0.39078097475044
20 0.535640035640036
25 0.765266384953186
30 0.929889298892989
35 0.970588235294118
40 1
45 1
50 1
};
\addlegendentry{$\alpha=1.0$}
\addplot [semithick, darkcyan32144140, mark=*, mark size=3, mark options={solid}]
table {%
10 0.378246753246753
15 0.455882352941176
20 0.780093165141397
25 0.916967509025271
30 1
35 1
40 1
45 1
50 1
};
\addlegendentry{$\alpha=1.5$}
\addplot [semithick, mediumseagreen53183120, mark=*, mark size=3, mark options={solid}]
table {%
10 0.434664246823956
15 0.551724137931034
20 0.892128279883382
25 0.958333333333333
30 1
35 1
40 1
45 1
50 1
};
\addlegendentry{$\alpha=2.0$}
\addplot [semithick, yellowgreen14421467, mark=*, mark size=3, mark options={solid}]
table {%
10 0.53648863035431
15 0.653846153846154
20 0.947368421052632
25 0.958333333333333
30 1
35 1
40 1
45 1
50 1
};
\addlegendentry{$\alpha=2.5$}
\addplot [semithick, gold25323136, mark=*, mark size=3, mark options={solid}]
table {%
10 0.524590163934426
15 0.711656441717791
20 1
25 1
30 1
35 1
40 1
45 1
50 1
};
\addlegendentry{$\alpha=3.0$}
\nextgroupplot[
legend cell align={left},
legend style={fill opacity=0.8, draw opacity=1, text opacity=1, draw=lightgray204},
tick align=outside,
tick pos=left,
title={\Large Median Ratio $J_{p}(\Tilde{\matr{L}}_{p}) / J_{p}(\matr{L}_{p}^{\star})$},
x grid style={darkgray176},
xlabel={\Large Number of Observations $n$},
xmajorgrids,
xmin=8, xmax=52,
xtick style={color=black},
y grid style={darkgray176},
ylabel={},
ymajorgrids,
ymin=0.7,ymax=100,
ytick style={color=black},
ymode=log,
ticklabel style={font=\large},
]
\addplot [semithick, darkslateblue6857130, mark=*, mark size=3, mark options={solid}]
table {%
10 32.79470676628
15 30.9065576059772
20 21.2479185312702
25 12.4062000286052
30 5.21363685651909
35 2.91632894858963
40 1.88998218508278
45 1.29464047093277
50 1
};

\addplot [semithick, darkslateblue48103141, mark=*, mark size=3, mark options={solid}]
table {%
10 52.400715769002
15 27.2418860890177
20 11.6706444740729
25 5.36784916686556
30 3.33274869176373
35 1.86841918500838
40 1
45 1
50 1
};

\addplot [semithick, darkcyan32144140, mark=*, mark size=3, mark options={solid}]
table {%
10 73.131906745378
15 33.1858392677387
20 5.97901280966285
25 3.30082724322275
30 1
35 1
40 1
45 1
50 1
};

\addplot [semithick, mediumseagreen53183120, mark=*, mark size=3, mark options={solid}]
table {%
10 67.8826440683716
15 27.4147374923442
20 6.10244546314748
25 2.8285681623844
30 1
35 1
40 1
45 1
50 1
};

\addplot [semithick, yellowgreen14421467, mark=*, mark size=3, mark options={solid}]
table {%
10 78.0629977847221
15 19.6890253920835
20 3.26257477319654
25 1.25806153509121
30 1
35 1
40 1
45 1
50 1
};

\addplot [semithick, gold25323136, mark=*, mark size=3, mark options={solid}]
table {%
10 72.9519665572188
15 22.3265969519041
20 1
25 1
30 1
35 1
40 1
45 1
50 1
};

\end{groupplot}

\end{tikzpicture}}}\vspace{-.1cm}
  \caption{Performance of GL-SigRep under partial observations: [Left] F1 score between $\widetilde{\L}_{p}$ and $\L^{\star}_{p}$; [Right] Ratio $J_{p}(\widetilde{\matr{L}}_{p}) / J_{p}(\matr{L}_{p}^{\star})$.}\vspace{-.5cm}
  \label{fig:exp1}
\end{figure}
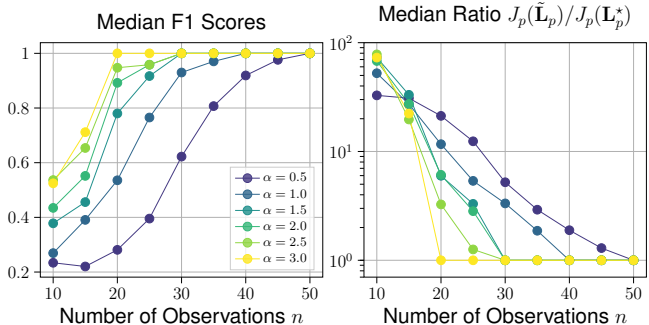

\vspace{.1cm}\noindent
\textbf{Robustness Against $n$.} The first set of experiment considers synthetically generated graph signals data. We focus on the performance of GL-SigRep \cite{dong2016smooth} in recovering the underlying binary graph from partially observed smooth signals.
Our setup consists of an unweighted graph $\mathcal{G}$ of $N = 50$ nodes independently sampled from an Erdos-Renyi (ER) model with edge probability of $p = 0.2$. The graph signal $\y_m$ is generated by the heat diffusion process \cite{thanou2017diffusion} as $\y_m = {\rm exp}(-\alpha \matr{L})\matr{x}_{m}$, in which the matrix exponential acts as a low-pass filter. The parameter $\alpha > 0$ controls the low-pass-ness of the graph filter; the higher $\alpha$ is, the more smooth the resulting signals are. The excitation signal $\matr{x}_{m}$ is sampled from a standard Gaussian distribution. In this experiment, we generate $M = 200$ partial signals $\matr{y}_{o, 1}, ..., \matr{y}_{o, M}$ for different number of observations $n \in \{10, 15, ..., 50\}$, and solve \eqref{eq:gl_p} with $\lambda = 2$.

Fig.~\ref{fig:exp1} reports the results from $100$ random trials, where we show the median F1 scores of $\L_p^\star$ compared with the observed part of $\L^\star$ on the [Left] plot, and the ratio $\frac{J_{p}(\widetilde{\matr{L}}_{p})}{J_{p}(\matr{L}^{\star}_{p})}$ on the [Right] plot. Even when $n \ll N$, the graph learned from \eqref{eq:gl_p} remains accurate compared to the benchmark with \eqref{eq:gl} and the ratio $\frac{J_{p}(\widetilde{\matr{L}}_p)}{J_{p}(\matr{L}^{\star}_p)}$ remains close to $1$, thus validating Theorem \ref{thm:main}. We can also see that the more low-pass the graph signals are, a good F1 score is achieved with a smaller number of observations $n$. In addition, while we do not have a fine-grained control over the magnitude of non-low-frequency component $\matr{y}_{m}^{\perp}$ in \eqref{eq:nonideal_model}, as the signal becomes more low-pass, $\matr{y}_{m}^{\perp}$ becomes smaller and GL-SigRep performance is less affected by these residual signal components. The simulation results show that Theorem \ref{thm:main} can capture the behavior of the optimal solutions of GL-SigRep, with respect to the number of hidden nodes and the smoothness of graph signals.

\vspace{.1cm}\noindent
\textbf{Comparison with \cite{buciulea2022learning} on Synthetic Data.} The second experiment demonstrates that under suitable conditions, the hidden-node agnostic algorithm GL-SigRep remains competitive compared to state-of-the-art algorithms tailor made for the partial observation scenarios. We consider as benchmark the GSm methods in \cite{buciulea2022learning} which enhances GL-SigRep by modeling the partial observation with low-rankness and sparsity.
% Specifically,
% \begin{align}
%      \min_{\hat{\matr{L}}_{p}, \matr{K}, r} \quad & \textstyle\alpha(\sum_{m=1}^{M}\matr{y}_{o, m}^{\top}\hat{\matr{L}}_{p}\matr{y}_{o, m} + 2{\rm Tr}(\matr{K}) + r)\\
%     &+ \textstyle\beta ||\hat{\matr{L}}_{p}||_{\rm F}^{2} + \gamma_{\star}||\matr{K}||_{\star} + \gamma_{2, 1}||\matr{K}||_{2, 1} \nonumber \\
%     {\rm s.t.} \quad &\textstyle \sum_{m=1}^{M}\matr{y}_{o, m}^{\top}\hat{\matr{L}}_{p}\matr{y}_{o, m} + 2{\rm Tr}(\matr{K}) + r \geq 0, r \geq 0, \nonumber \\
%     &\textstyle{\rm Tr}(\hat{\matr{L}}_{p}) = n, \hat{\matr{L}}_{p} \in \mathcal{L}_{n}. \nonumber
% \end{align}
Note in our implementation, we ignored the log-barrier $-\log({\rm Diag}(\hat{\matr{L}}_{p}))$ regularizer, and included the constraint ${\rm Tr}(\hat{\matr{L}}_{p}) = n$ as in GL-SigRep. This configuration seems to provide a more stable performance for GSm. We compare all the three variants of GSm in \cite{buciulea2022learning}: GSm-LR models the low-rank-ness of the observed-hidden part by a nuclear norm regularizer; GSm-GL models the column sparsity of the observed-hidden part with a group Lasso penalty.

This experiment uses an unweighted graph of $N = 20$ nodes, sampled from an $k$-nearest neighbors graph model with $k=5$, constructed by sampling nodes from a two-dimensional uniform distribution $U(0, 1)$, and a stochastic block model of two equal-sized communities, with inter-cluster probability of $p_{\rm out} = 0.05$ and intra-cluster probability of $p_{\rm in} = 0.6$. $M = 200$ signals are generated by a low-pass filter as
$\matr{y}_{m} = (\matr{I}+2.5\matr{L})^{-1}\matr{x}_{m}$, where excitation signal $\matr{x}_{m}$ is standard Gaussian.
Fig.~\ref{fig:exp2} shows the results from $100$ random trials. We observe that the performance of GL-SigRep as well as GSm-LR/GSm-GL are similar when $n$ is sufficiently large. Our result suggests that when the graph signals are sufficiently low-pass, GL-SigRep is implicitly robust. In this case, extra modeling of hidden nodes' influence based on low-rankness or sparsity might not bring significant improvements.

\vspace{.1cm}\noindent
{\bf Experiments on Real Data.} The last set of experiments examines the performance of GL-SigRep on real data with partial observations.  In the first experiment, we use a Twitter interaction graph between $N = 30$ members of the 117th United States Congress \cite{fink2023centrality}, where an edge indicates that one person follows the other on Twitter. We generate $M = 100$ low-pass graph signals $\y_{m} = (\I + \L)^{-1}\x_{m}$ on this graph. In the second experiment, we use the average monthly temperature at $N = 40$ stations in Switzerland during 1981-2010 \cite{meteoswiss2025} as our graph signals $\y_1, ..., \y_{12} \in \bbR^{40}$. Since there is no ground-truth graph, we follow \cite{dong2016smooth, buciulea2022learning} and construct an undirected network of stations in which a pair of nodes are connected if their altitude difference is less than $300$ meters. For evaluation, we calculate the F1 score of the observed subgraph estimated by GL-SigRep under full observation vs. the constructed ground-truth, and similarly for GL-SigRep under partial observation. The results are reported in Figure \ref{fig:exp3}, where we observe that GL-SigRep is implicitly robust against partial observation in these two datasets: its performance in recovering ground-truth observed subgraph using partial signals and using full signals are quite close. This confirms the applicability of our developed theory in practical scenarios.

\begin{figure}[t!]
\centering
  {\sf \resizebox{1.\linewidth}{!}{% This file was created with matplot2tikz v0.4.0.
\begin{tikzpicture}

\definecolor{darkgray176}{RGB}{176,176,176}
\definecolor{goldenrod1911910}{RGB}{191,191,0}
\definecolor{green01270}{RGB}{0,127,0}
\definecolor{lightgray204}{RGB}{204,204,204}

\begin{groupplot}[group style={group size=2 by 1}]

\nextgroupplot[
legend cell align={left},
legend style={
  fill opacity=0.8,
  draw opacity=1,
  text opacity=1,
  at={(0.91,0.5)},
  anchor=east,
  draw=lightgray204
},
tick align=outside,
tick pos=left,
x grid style={darkgray176},
xlabel={\Large Number of Observations $n$},
xmajorgrids,
xmin=7.4, xmax=20.6,
xtick style={color=black},
y grid style={darkgray176},
ylabel={\Large Median F1 Scores},
ymajorgrids,
ymin=0.427211394302849, ymax=0.890629685157421,
ytick style={color=black},
ticklabel style={font=\large},
]
\addplot [semithick, red, mark=*, mark size=3, mark options={solid}]
table {%
8 0.869565217391304
10 0.846153846153846
12 0.792452830188679
14 0.821917808219178
16 0.808510638297872
18 0.819672131147541
20 0.818181818181818
};
\addlegendentry{GL-SigRep}
\addplot [semithick, blue, mark=+, mark size=3, mark options={solid}]
table {%
8 0.526315789473684
10 0.448275862068966
12 0.5
14 0.495867768595041
16 0.481012658227848
18 0.492610837438424
20 0.49802371541502
};
\addlegendentry{GSm}
\addplot [semithick, green01270, mark=asterisk, mark size=3, mark options={solid}]
table {%
8 0.869565217391304
10 0.764705882352941
12 0.792452830188679
14 0.821917808219178
16 0.808510638297872
18 0.819672131147541
20 0.818181818181818
};
\addlegendentry{GSm-LR}
\addplot [semithick, goldenrod1911910, mark=triangle*, mark size=3, mark options={solid}]
table {%
8 0.75
10 0.7
12 0.745098039215686
14 0.836363636363636
16 0.829268292682927
18 0.841121495327103
20 0.818181818181818
};
\addlegendentry{GSm-GL}

\nextgroupplot[
legend cell align={left},
legend style={
  fill opacity=0.8,
  draw opacity=1,
  text opacity=1,
  at={(0.91,0.5)},
  anchor=east,
  draw=lightgray204
},
tick align=outside,
tick pos=left,
x grid style={darkgray176},
xlabel={\Large Number of Observations $n$},
xmajorgrids,
xmin=7.4, xmax=20.6,
xtick style={color=black},
y grid style={darkgray176},
ymajorgrids,
ymin=0.426666666666667, ymax=0.817777777777778,
ytick style={color=black},
ticklabel style={font=\large},
]
\addplot [semithick, red, mark=*, mark size=3, mark options={solid}]
table {%
8 0.8
10 0.777777777777778
12 0.757358490566038
14 0.738461538461539
16 0.746760701984583
18 0.731707317073171
20 0.725429806012687
};
\addplot [semithick, blue, mark=+, mark size=3, mark options={solid}]
table {%
8 0.444444444444444
10 0.474576271186441
12 0.465116279069767
14 0.470588235294118
16 0.481012658227848
18 0.477611940298507
20 0.48
};
\addplot [semithick, green01270, mark=asterisk, mark size=3, mark options={solid}]
table {%
8 0.780193236714976
10 0.757352941176471
12 0.752358490566038
14 0.728501228501228
16 0.729166666666667
18 0.724137931034483
20 0.718097813838179
};
\addplot [semithick, goldenrod1911910, mark=triangle*, mark size=3, mark options={solid}]
table {%
8 0.705882352941177
10 0.692307692307692
12 0.688577586206897
14 0.706787330316742
16 0.722222222222222
18 0.733160036166365
20 0.728747864932335
};
\end{groupplot}

\end{tikzpicture}}}\vspace{-.2cm}
  \caption{Comparing GL-SigRep vs.~GSm/GSm-LR/GSm-GL: [Left] $k$-nearest neighbors graph and [Right] stochastic block model.}
  \label{fig:exp2}
\end{figure}
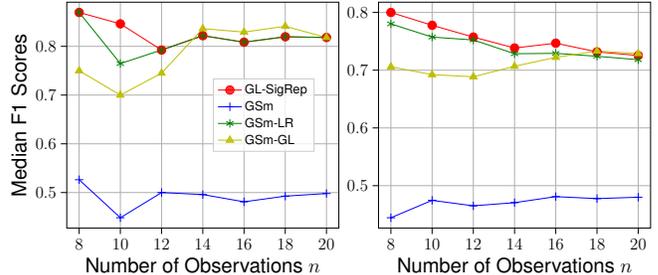

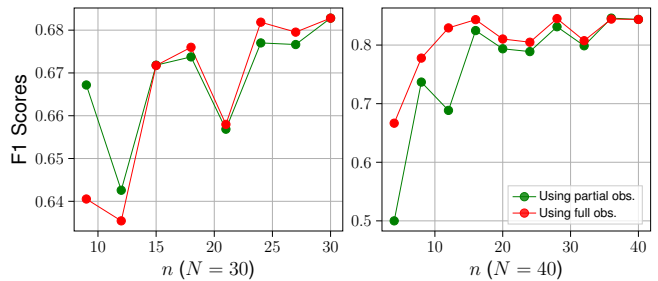
\begin{figure}[t!]
\centering
  {\sf \resizebox{1.\linewidth}{!}{% This file was created with matplot2tikz v0.4.0.
\begin{tikzpicture}

\definecolor{darkgray176}{RGB}{176,176,176}
\definecolor{green01270}{RGB}{0,127,0}
\definecolor{lightgray204}{RGB}{204,204,204}

\begin{groupplot}[group style={group size=2 by 1}]

\definecolor{darkgray176}{RGB}{176,176,176}
\definecolor{green01270}{RGB}{0,127,0}
\definecolor{lightgray204}{RGB}{204,204,204}

\nextgroupplot[
legend cell align={left},
legend style={
  fill opacity=0.8,
  draw opacity=1,
  text opacity=1,
  at={(0.03,0.97)},
  anchor=north west,
  draw=lightgray204
},
tick align=outside,
tick pos=left,
x grid style={darkgray176},
xlabel={\Large $n$ ($N = 30$)},
xmajorgrids,
xmin=7.95, xmax=31.05,
xtick style={color=black},
y grid style={darkgray176},
ylabel={\Large F1 Scores},
ymajorgrids,
ymin=0.633063474003196, ymax=0.685177537788156,
ytick style={color=black},
ticklabel style={font=\large},
]
\addplot [semithick, green01270, mark=*, mark size=3, mark options={solid}]
table {%
9 0.667186010056461
12 0.642589084137163
15 0.671819058365031
18 0.673749210712691
21 0.656850743210689
24 0.677013652708164
27 0.676638285271714
30 0.682808716707022
};
\addplot [semithick, red, mark=*, mark size=3, mark options={solid}]
table {%
9 0.64054009427515
12 0.635432295084331
15 0.671746673490155
18 0.67599407853133
21 0.65797997367173
24 0.681861479778213
27 0.679536350596325
30 0.682808716707022
};

\nextgroupplot[
legend cell align={left},
legend style={
  fill opacity=0.8,
  draw opacity=1,
  text opacity=1,
  at={(0.97,0.03)},
  anchor=south east,
  draw=lightgray204
},
tick align=outside,
tick pos=left,
x grid style={darkgray176},
xlabel={\Large $n$ ($N = 40$)},
xmajorgrids,
xmin=2.2, xmax=41.8,
xtick style={color=black},
y grid style={darkgray176},
ylabel={},
ymajorgrids,
ymin=0.482701149425287, ymax=0.863275862068966,
ytick style={color=black},
ticklabel style={font=\large},
]
\addplot [semithick, green01270, mark=*, mark size=3, mark options={solid}]
table {%
4 0.5
8 0.736842105263158
12 0.688524590163934
16 0.824742268041237
20 0.793650793650794
24 0.788888888888889
28 0.831275720164609
32 0.798780487804878
36 0.845977011494253
40 0.84375
};
\addlegendentry{Using partial obs.}
\addplot [semithick, red, mark=*, mark size=3, mark options={solid}]
table {%
4 0.666666666666667
8 0.777777777777778
12 0.829268292682927
16 0.843373493975904
20 0.810344827586207
24 0.804878048780488
28 0.845188284518828
32 0.807453416149068
36 0.844339622641509
40 0.84375
};
\addlegendentry{Using full obs.}
\end{groupplot}

\end{tikzpicture}}}\vspace{-.2cm}
  \caption{Performance of GL-SigRep under partial observation vs. full observation: [Left] Twitter interaction network of U.S. Congress members and [Right] meteorological network in Switzerland.}\vspace{-.3cm}
  \label{fig:exp3}
\end{figure}

\vspace{-.1cm}
\section{Conclusion}\vspace{-.1cm}
This paper shows that GL-SigRep from \cite{dong2016smooth} is implicitly robust to partial observations when the underlying graph signals are sufficiently low pass. We demonstrate both theoretically and empirically that the method remains competitive compared to state-of-the-art graph learning methods tailor made for partial observations. We anticipate that our results can be extended to analyze other paradigms involving graph learning with partial observations.

\clearpage

\bibliographystyle{IEEEtran}
\bibliography{strings}

\clearpage

\section*{\texorpdfstring{Appendix A: Proof of Theorem \ref{thm:main}}{Appendix A: Proof of Theorem 1}}
\begin{Lemma} \label{lemma:rip}
Consider a random partial observation
\[
\mathcal{V}_{o} = \{s(1), ..., s(n)\},
\]
which is sampled uniformly without replacement from the node indices $\{1, ..., N\}$. The set $\mathcal{V}_{o}$ is encoded in a matrix $\E_{o} \in \{0, 1\}^{n \times N}$, and the partial graph signal is given by $\y_{o} = \E_{o}\y \in \mathbb{R}^{n}$. For any $\delta \in (0, 1)$, there exists $t \in (0, 1)$ such that with probability at least $1 - \delta$,
\begin{align*}
    \y_{o}^{\top}(\matr{E}_{o}\matr{L}\matr{E}_{o}^{\top})\y_{o} \leq (1+t)\frac{n}{N}\frac{\sigma_{\max}(\L)}{\sigma_{\min}^{+}(\L)}\y^{\top}\L\y,~\forall \y \in \text{\rm span}(\V_{K}),
\end{align*}
provided that
\begin{align*}
    \frac{n}{N} \geq \frac{3}{t^2}\max_{1 \leq i \leq N}\|\V_{K}^{\top}\e_{i}\|_{2}^{2}\ln\left(\frac{K}{\delta}\right).
\end{align*}
\end{Lemma}
\begin{proof}
To begin, we notice that
\begin{align*}
    \y_{o}^{\top}(\matr{E}_{o}\L\matr{E}_{o}^{\top})\y_{o}
    &= \y^{\top}\E_{o}^{\top}\matr{E}_{o}\L\matr{E}_{o}^{\top}\E_{o}\y\\
    &\leq \|\L\|_{2}(\y^{\top}\E_{o}^{\top}\matr{E}_{o}\matr{E}_{o}^{\top}\E_{o}\y)\\
    &= \|\L\|_{2}(\y^{\top}\E_{o}^{\top}\E_{o}\y).
\end{align*}
Let $\Z_{\ell} := \frac{N}{n}\V_{K}^{\top}\e_{s(\ell)}\e_{s(\ell)}^{\top}\V_{K}$, which is sampled from the set $\mathcal{Z} = \{\frac{N}{n}\matr{V}_{K}^{\top}\matr{e}_i\matr{e}_{i}^{\top}\matr{V}_{K}\}_{i=1}^{N}$, where $\matr{e}_{i} \in \bbR^{N}$ is the $i$-th standard basis vector. A sum of $n$ symmetric matrices $\matr{Z}_{1}, ..., \matr{Z}_{n}$ is
\begin{align*}
    \X := \sum_{\ell = 1}^{n}\Z_{\ell}
    &= \frac{N}{n}\V_{K}^{\top}\left(\sum_{\ell = 1}^{n}\e_{s(\ell)}\e_{s(\ell)}^{\top}\right)\V_{K}\\
    &= \frac{N}{n}\V_{K}^{\top}\E_{o}^{\top}\E_{o}\V_{K}.
\end{align*}
Note the followings:
\begin{align*}
    &\textstyle \bbE[\Z_{1}] = \frac{N}{n}\V_{K}^{\top}\bbE[\e_{s(1)}\e_{s(1)}^{\top}]\V_{K} = \frac{1}{n}\V_{K}^{\top}\V_{K} = \frac{1}{n}\I,\\
    &\textstyle \mu_{\max} := n\lambda_{\max}\left(\bbE[\Z_{1}]\right) = 1,\\
    &\textstyle \max_{\Z \in {\cal Z}}\lambda_{\max}(\Z) = \frac{N}{n}\max_{\ell=1, ..., N}\lambda_{\max}(\V_{K}^{\top}\e_{s(\ell)}\e_{s(\ell)}^{\top}\V_{K})\\
    &= \textstyle \frac{N}{n}\max_{i}\|\V_{K}^{\top}\e_{i}\|_{2}^{2}.
\end{align*}
Then, the matrix Chernoff's bound for the case of uniformly random sampling without replacement \cite[Theorem 2.1]{tropp2011hadamard} states that for any $\delta \in (0, 1)$, with probability at least $1 - \delta$,
\begin{align}
    \|\X\|_{2} \leq 1+t \label{eq:bound},
\end{align}
provided that
\begin{align*}
    \frac{n}{N} \geq \frac{3}{t^2}\max_{1 \leq i \leq N}\|\V_{K}^{\top}\e_{i}\|_{2}^{2}\ln\left(\frac{K}{\delta}\right).
\end{align*}
Observe that \eqref{eq:bound} implies that
\begin{align*}
    \frac{N}{n}\|\E_{o}\V_{K}\boldsymbol{\upsilon}\|_{2}^{2} \leq (1+t)\|\boldsymbol{\upsilon}\|_{2}^{2},~\forall \boldsymbol{\upsilon} \in \mathbb{R}^{K}
\end{align*}
which means for any $\y \in \text{span}(\V_{K})$,
\begin{align*}
    \frac{N}{n}\|\E_{o}\y\|_{2}^{2} \leq (1+t)\|\y\|_{2}^{2},
\end{align*}
Hence, with high probability, $\y_{o}^{\top}\L_{oo}\y_{o}$ can be bounded as
\begin{align*}
    \y_{o}^{\top}(\matr{E}_{o}\L\matr{E}_{o}^{\top})\y_{o}
    &\leq (1+t)\frac{n}{N}\|\L\|_{2}\|\y\|_{2}^{2}\\
    &\leq (1+t)\frac{n}{N}\|\L\|_{2}\|(\L^{1/2})^{\dagger}\|_{2}^{2}\|\L^{1/2}\y\|_{2}^{2}\\
    &=(1+t)\frac{n}{N}\frac{\sigma_{\max}(\L)}{\sigma_{\min}^{+}(\L)}\y^{\top}\L\y.
\end{align*}
\end{proof}

\mainresult*

\begin{proof}
Let $\textstyle C(t) := (1+t)\frac{\sigma_{\max}(\L)}{\sigma_{\min}^{+}(\L)}$. Applying Lemma \ref{lemma:rip} yields:
\begin{align*}
    &J_{f}( \Lfull^\star ) \leq J_{f}( \widehat{\Lfull} )\\
    &= \frac{N}{n} \sum_{m=1}^{M} {\bf y}_{o,m}^\top \Lpart^\star {\bf y}_{o,m}\\
    &= \frac{N}{n} J_{p} ( \Lpart^\star )\\
    &\leq \frac{N}{n} J_{p} ( \widetilde{\L}_{p} )\\
    &= \frac{N}{ {\rm Tr}( \Lfull_{oo}^\star ) + {\bf 1}^\top \Lfull_{oh}^\star {\bf 1} } \left( J_{p}( {\bf E}_o \Lfull^\star {\bf E}_o^\top ) + J_{p}( {\rm Diag}( \Lfull_{oh}^\star {\bf 1}) ) \right) \\
    & \leq \frac{n}{ {\rm Tr}( \Lfull_{oo}^\star ) + {\bf 1}^\top \Lfull_{oh}^\star {\bf 1} }C(t) J_{f}( \Lfull^\star )\\
    &+ \frac{N}{n}\frac{n}{ {\rm Tr}( \Lfull_{oo}^\star ) + {\bf 1}^\top \Lfull_{oh}^\star {\bf 1} }J_{p}( {\rm Diag}( \Lfull_{oh}^\star {\bf 1}) ).
\end{align*}
With $\frac{n}{ {\rm Tr}( \Lfull_{oo}^\star ) + {\bf 1}^\top \Lfull_{oh}^\star {\bf 1} } \leq \frac{1}{c}$ and $-\L^{\star}_{oh}\matr{1} = {\cal O}(\epsilon)$, this gives an approximation bound
\begin{align*}
    J_{f}( {\Lfull}^\star ) \leq J_{f}( \widehat{\Lfull} ) \leq \frac{C(t)}{c}J_{f}( \Lfull^\star ) + {\cal O}\left(\frac{N\epsilon}{nc}\right).
\end{align*}
Conversely, we can also bound that 
\begin{align*}
    &J_{p}( \Lpart^\star )\\
    &\leq J_{p} ( \widetilde{\L}_{p} )\\
    &= \frac{n}{ {\rm Tr}( \Lfull_{oo}^\star ) + {\bf 1}^\top \Lfull_{oh}^\star {\bf 1} } \left( J_{p}( {\bf E}_o \Lfull^\star {\bf E}_o^\top ) + J_{p}( {\rm Diag}( \Lfull_{oh}^\star {\bf 1}) ) \right) \\
    & \leq \frac{n}{ {\rm Tr}( \Lfull_{oo}^\star ) + {\bf 1}^\top \Lfull_{oh}^\star {\bf 1} } \left( C(t) \frac{n}{N} J_{f}( \Lfull^\star ) + J_{p}( {\rm Diag}( \Lfull_{oh}^\star {\bf 1}) ) \right) \\
    & \leq \frac{n}{ {\rm Tr}( \Lfull_{oo}^\star ) + {\bf 1}^\top \Lfull_{oh}^\star {\bf 1} } \left( C(t) \frac{n}{N} J_{f}( \widehat{\Lfull} ) + J_{p}( {\rm Diag}( \Lfull_{oh}^\star {\bf 1}) ) \right) \\
    & = \frac{n}{ {\rm Tr}( \Lfull_{oo}^\star ) + {\bf 1}^\top \Lfull_{oh}^\star {\bf 1} } \left( C(t) J_{p}( \Lpart^\star ) + J_{p}( {\rm Diag}( \Lfull_{oh}^\star {\bf 1}) ) \right).
\end{align*}
Therefore,
\begin{align*}
    J_{p}( \Lpart^\star ) \leq J_{p} ( \widetilde{\L}_{p} ) \leq \frac{C(t)}{c}J_{p}( \Lpart^\star ) + {\cal O}\left(\frac{\epsilon}{c}\right).
\end{align*}
\end{proof}

\section*{\texorpdfstring{Appendix B: Proof of Corollary \ref{cor:nonideal}}{Appendix B: Proof of Corollary 1}}

\begin{Lemma}
\label{lemma:nonideal}
Assume the non-ideal low-pass signal model as in \eqref{eq:nonideal_model}. Then, the following inequality holds:
\begin{align*}
    \textstyle \matr{y}_{o}^{\top}(\matr{E}_{o}\matr{L}\matr{E}_{o}^{\top})\matr{y}_{o}
    &\leq (\matr{y}_{o}^{\parallel})^{\top}(\matr{E}_{o}\matr{L}\matr{E}_{o}^{\top})\matr{y}_{o}^{\parallel}\\
    &+ \textstyle \mathcal{O}(||\matr{L}||_{2}\eta_{K}H^{2}M^{2}).
\end{align*}
\end{Lemma}

\begin{proof}
One can expand the quadratic form $\matr{y}_{o}^{\top}\matr{E}_{o}\matr{L}\matr{E}_{o}^{\top}\matr{y}_{o}$ of the partial signal $\matr{y}_{o} = \matr{E}_{o}\matr{y} \in \bbR^{n}$, then apply Cauchy-Schwarz inequality and norm equivalent to obtain
\begin{align*}
    &\mathbf{y}_{o}^{\top}\matr{E}_{o}\matr{L}\matr{E}_{o}^{\top}\mathbf{y}_{o}\\
    &= (\mathbf{y}_{o}^{\parallel})^{\top}\matr{E}_{o}\matr{L}\matr{E}_{o}^{\top}\mathbf{y}_{o}^{\parallel} + 2(\mathbf{y}^{\parallel})^{\top}\matr{E}_{o}^{\top}\matr{E}_{o}\mathbf{L}\matr{E}_{o}^{\top}\matr{E}_{o}\mathbf{y}^{\perp}\\
    &+ (\mathbf{y}^{\perp})^{\top}\matr{E}_{o}^{\top}\matr{E}_{o}\mathbf{L}\matr{E}_{o}^{\top}\matr{E}_{o}\mathbf{y}^{\perp}\\
    &\leq (\mathbf{y}_{o}^{\parallel})^{\top}\matr{E}_{o}\matr{L}\matr{E}_{o}^{\top}\mathbf{y}_{o}^{\parallel} + 2||\matr{L}||_{2}||\matr{E}_{o}^{\top}\matr{E}_{o}\matr{y}^{\perp}||_{2}||\matr{E}_{o}^{\top}\matr{E}_{o}\matr{y}^{\parallel}||_{2}\\
    &+ ||\matr{L}||_{2}||\matr{E}_{o}^{\top}\matr{E}_{o}\matr{y}^{\perp}||_{2}^{2}\\
    &\leq (\mathbf{y}_{o}^{\parallel})^{\top}\matr{E}_{o}\matr{L}\matr{E}_{o}^{\top}\mathbf{y}_{o}^{\parallel} + 2||\matr{L}||_{2}||\matr{y}^{\perp}||_{2}||\matr{y}^{\parallel}||_{2} + ||\matr{L}||_{2}||\matr{y}^{\perp}||_{2}^{2}.
\end{align*}
We want to bound $||\matr{y}^{\perp}||_{2}$ as well as $||\matr{y}^{\parallel}||_{2}$ based on the properties of low-pass filtering process that gives rise to the graph signal $\matr{y}$. With $\hat{x} = [\hat{x}_1, ..., \hat{x}_{N}] \in \bbR^{N}$ such that $\hat{\x} := \matr{V}^{\top}\matr{x}$, we have that
\begin{align}
    ||\matr{y}^{\parallel}||^{2}_{2}
    &=\textstyle \sum_{i=1}^{K}|h(\lambda_i)|^{2}|\hat{x}_{i}|^{2}\\
    &\leq \textstyle (H_{\max})^{2}\sum_{i=1}^{K}|\hat{x}_{i}|^{2}\\
    &\leq (H_{\max})^{2}||\hat{\matr{x}}||_{2}^{2}\\
    &\leq (H_{\max})^{2}M^{2}.
\end{align}
Let us denote
\begin{align*}
    \textstyle H_{-K}^{\max} := \max_{i=K+1, ..., N} |h(\lambda_i)|,~H_{K}^{\min} := \min_{i=1, ..., K} |h(\lambda_i)|.
\end{align*}
By the low-pass property in Definition \ref{def:lpf}, we similarly have
\begin{align}
    ||\matr{y}^{\perp}||_{2}^{2}
    &\leq \textstyle (H_{-K}^{\max})^{2}\sum_{i=K+1}^{N}|\hat{x}_{i}|^{2}\\
    &\leq \textstyle\eta_{K}^{2}(H_{K}^{\min})^{2}\sum_{i=K+1}^{N}|\hat{x}_{i}|^{2}\\\
    &\leq \eta_{K}^{2}H^{2}||\hat{\x}||_{2}^{2}\\
    &\leq \eta_{K}^{2}H^{2}M^{2}.
\end{align}
Therefore, we can conclude that
\begin{align}
    &\matr{y}_{o}^{\top}\matr{E}_{o}\matr{L}\matr{E}_{o}^{\top}\matr{y}_{o}\\
    &\leq (\matr{y}_{o}^{\parallel})^{\top}\matr{E}_{o}\matr{L}\matr{E}_{o}^{\top}\matr{y}_{o}^{\parallel}\\
    &+ 2||\matr{L}||_{2}\eta_{K}H^{2}M^{2} + ||\matr{L}||_{2}\eta_{K}^{2}H^{2}M^{2}\\
    &= (\matr{y}_{o}^{\parallel})^{\top}\matr{E}_{o}\matr{L}\matr{E}_{o}^{\top}\matr{y}_{o}^{\parallel} + \mathcal{O}(||\matr{L}||_{2}\eta_{K}H^{2}M^{2}).
\end{align}
\end{proof}

\cornonideal*
\begin{proof}
Combining Lemma \ref{lemma:rip} and \ref{lemma:nonideal}, we obtain the following result: for any $\delta \in (0, 1)$, there exists $t \in (0, 1)$ such that with probability at least $1-\delta$, for any $\matr{y} \in \bbR^{N}$ we have
\begin{align}
    \matr{y}_{o}^{\top}(\matr{E}_{o}\matr{L}\matr{E}_{o}^{\top})\matr{y}_{o}
    &\leq (\y_{o}^{\parallel})^{\top}(\E_{o}\L\E_{o})^{\top}\y_{o}^{\parallel} + {\cal O}(||\L||_{2}\eta_{K}H^{2}M^{2}) \nonumber\\
    &\leq (1+t)\frac{n}{N}\frac{\sigma_{\max}(\matr{L})}{\sigma_{\min}^{+}(\matr{L})}(\matr{y}^{\parallel})^{\top}\matr{L}\matr{y}^{\parallel} \nonumber\\
    &+ \mathcal{O}(||\matr{L}||_{2}\eta_{K}H^2M^2),\label{ineq:rip_combined}
\end{align}
provided that
\begin{align*}
    \frac{n}{N} \geq \frac{3}{t^2}\max_{1 \leq i \leq N}\|\V_{K}^{\top}\e_{i}\|_{2}^{2}\ln\left(\frac{K}{\delta}\right).
\end{align*}
Moreover,
\begin{align}
    \y^{\top}\L\y
    &= (\y^{\parallel})^{\top}\L\y^{\parallel} + 2(\y^{\parallel})^{\top}\L\y^{\perp} + (\y^{\perp})^{\top}\L\y^{\perp}\\
    &= (\y^{\parallel})^{\top}\L\y^{\parallel} +(\y^{\perp})^{\top}\L\y^{\perp}\\
    &\geq (\y^{\parallel})^{\top}\L\y^{\parallel},
\end{align}
where we used the fact that $(\y^{\parallel})^{\top}\L\y^{\perp} = 0$ due to the non-ideal low-pass signal model \eqref{eq:nonideal_model}, and $(\y^{\perp})^{\top}\L\y^{\perp} \geq 0$ as the Laplacian $\L$ is a positive semi-definite matrix.
This allows us to relax \eqref{ineq:rip_combined} as
\begin{align}
    \matr{y}_{o}^{\top}(\matr{E}_{o}\matr{L}\matr{E}_{o}^{\top})\matr{y}_{o}
    &\leq (1+t)\frac{n}{N}\frac{\sigma_{\max}(\matr{L})}{\sigma_{\min}^{+}(\matr{L})}\matr{y}^{\top}\matr{L}\matr{y} \nonumber\\
    &+ \mathcal{O}(||\matr{L}||_{2}\eta_{K}H^2M^2),~\forall \matr{y} \in \bbR^{N}.\label{ineq:rip_combined_relaxed}
\end{align}
Let $\textstyle C(t) := (1+t)\frac{\sigma_{\max}(\L)}{\sigma_{\min}^{+}(\L)}$. Similar to Theorem \ref{thm:main}, we apply \eqref{ineq:rip_combined_relaxed} and obtain
\begin{align*}
    &J_{f}( \Lfull^\star ) \leq J_{f}( \widehat{\Lfull} )\\
    &\leq \frac{N}{ {\rm Tr}( \Lfull_{oo}^\star ) + {\bf 1}^\top \Lfull_{oh}^\star {\bf 1} }\Big( C(t) \frac{n}{N} J_{f}( \Lfull^\star ) + {\cal O}(||\L||_{2}\eta{K} H^2 M^2)\\
    &+ J_{p}( {\rm Diag}( \Lfull_{oh}^\star {\bf 1}) ) \Big)\\
    &= \frac{n}{ {\rm Tr}( \Lfull_{oo}^\star ) + {\bf 1}^\top \Lfull_{oh}^\star {\bf 1} } C(t) J_{f}( \Lfull^\star ) \nonumber\\
    &+ \frac{N}{n}\frac{n}{ {\rm Tr}( \Lfull_{oo}^\star ) + {\bf 1}^\top \Lfull_{oh}^\star {\bf 1} } \Big({\cal O}(||\L||_{2}\eta_{K} H^2 M^2) + J_{p}( {\rm Diag}( \Lfull_{oh}^\star {\bf 1}) ) \Big).
\end{align*}
With $\frac{n}{ {\rm Tr}( \Lfull_{oo}^\star ) + {\bf 1}^\top \Lfull_{oh}^\star {\bf 1} } \leq \frac{1}{c}$ and $-\L_{oh}^{\star}\matr{1} = {\cal O}(\epsilon)$, we have an approximation bound as
\begin{align*}
    J_{f}( {\Lfull}^\star ) \leq J_{f}( \widehat{\Lfull} ) \leq \frac{C(t)}{c}J_{f}( \Lfull^\star ) + {\cal O}\left(\frac{N(||\L||_{2}\eta_{K} H^2M^2 + \epsilon)}{nc}\right).
\end{align*}
Similarly, we can also develop the following bound: 
\begin{align*}
    J_{p}( \Lpart^\star ) \leq J_{p} ( \widetilde{\L}_{p} ) \leq \frac{C(t)}{c}J_{p}( \Lpart^\star ) + {\cal O}\left(\frac{||\L||_{2}\eta_{K}H^{2}M^{2} + \epsilon}{c}\right).
\end{align*}
\end{proof}

\end{document}